\documentclass[11pt]{article}
\usepackage[margin=1in]{geometry}
\usepackage{times}
\usepackage{amsmath,amssymb,amsthm}
\usepackage{latexsym, epic, eepic,graphics,graphicx}

\usepackage{graphicx,algorithm,algorithmic}

\usepackage{pdfsync}

\newcommand{\ignore}[1]{}

\def\Lone{$\ell_1$}
\def\epsilon{\varepsilon}
\def\zero{\mathbf{0}}
\def\cone{\textnormal{\texttt{cone}}}

\def\dist{\textnormal{\texttt{dist}}}
\def\reals{\mathbb{R}}
\def\X{\mathcal{X}}
\def\Y{\mathcal{Y}}
\def\A{\mathcal{A}}
\def\L{\mathcal{L}}
\def\lb{\boldsymbol{\ell}}
\def\xv{\mathbf{x}}
\def\x{\mathbf{x}}

\def\w{\mathbf{w}}
\def\uv{\mathbf{v}}
\def\wv{\mathbf{w}}

\def\H{\mathcal{H}}
\def\vv{\mathbf{v}}
\def\fv{\mathbf{f}}
\def\zv{\mathbf{z}}
\def\yv{\mathbf{y}}
\def\F{\mathcal{F}}
\def\K{\mathcal{K}}
\def\I{\mathbb{I}}

\newcommand{\regret}{\textnormal{Regret}}

\def\thb{\boldsymbol{\theta}}

\newtheorem{prop}{Proposition}
\newtheorem{lem}{Lemma}
\newtheorem{thm}{Theorem}
\newtheorem{defn}{Definition}
\newtheorem{claim}{Claim}

\title{Blackwell Approachability and Low-Regret Learning are Equivalent}

\author{
Jacob Abernethy\\
Computer Science Division\\
University of California, Berkeley\\
\texttt{\small jake@cs.berkeley.edu}
\and
Peter L. Bartlett\\
Computer Science Division and Department of Statistics\\
University of California, Berkeley\\
\texttt{\small bartlett@cs.berkeley.edu}
\and
Elad Hazan\\
Faculty of Industrial Engineering \& Management\\
Technion - Israel Institute of Technology\\
\texttt{\small ehazan@ie.technion.ac.il}
}
\date{}

\begin{document}

\maketitle

\begin{abstract}
	We consider the celebrated Blackwell Approachability Theorem for two-player games with vector payoffs. We show that Blackwell's result is equivalent, via  efficient reductions, to the existence of ``no-regret'' algorithms for Online Linear Optimization. Indeed, we show that any algorithm for one such problem can be efficiently converted into an algorithm for the other. We provide a useful application of this reduction: the first \emph{efficient} algorithm for calibrated forecasting.
\end{abstract}

\newpage

\section{Introduction}

A typical assumption in game theory, and indeed in most of economics, is that an agent's goal is to optimize a scalar-valued payoff function--a person's wealth, for example. Such scalar-valued utility functions are the basis for much work in learning and Statistics too, where one hopes to maximize prediction accuracy or minimize expected loss. Towards this end, a natural goal is to prove a guarantee on some algorithm's minimum expected payoff (or maximum reward).

In 1956, David Blackwell posed an intriguing question: what guarantee can we hope to achieve when playing a two-player game with a \emph{vector-valued payoff}, particularly when the opponent is potentially an adversary?
For the case of scalar payoffs, as in a two-player zero-sum game, we already have a concise guarantee by way of Von Neumann's minimax theorem: either player has a fixed oblivious strategy that is effectively the ``best possible'', in that this player could do no better even with knowledge of the opponent's randomized strategy in advance. This result is equivalent to strong duality for linear programming.

When our payoffs are non-scalar quantities, it does not make sense to ask ``can we earn at least $x$?''. Instead, we would like to ask ``can we guarantee that our vector payoff lies in some convex set $S$''? In this case, the story is more difficult, and Blackwell observed that an oblivious strategy does not suffice---in short, we do not achieve ``duality'' for vector-payoff games. What Blackwell was able to prove is that this negative result applies only for \emph{one-shot games}. In his celebrated Approachability Theorem \cite{blackwell_analog_1956}, one can achieve a duality statment \emph{in the limit} when the game is played repeatedly, where the player may learn from his opponent's prior actions. Blackwell actually constructed an algorithm (that is, an adaptive strategy) with the guarantee that the average payoff vector ``approaches'' $S$, hence the name of the theorem.

Blackwell Approachability has the flavor of learning in repeated games, a topic which has received much interest. In particular, there are a wealth of recent results on so-called \emph{no-regret learning algorithms} for making repeated decisions given an arbitrary (and potentially adversarial) sequence of cost functions. The first no-regret algorithm for a ``discrete action'' setting was given in a seminal paper by James Hannan in 1956 \cite{hannan1957approximation}.  That same year, David Blackwell pointed out \cite{blackwell1954controlled} that his Approachability result leads, as a special case, to an algorithm with essentially the same low-regret guarantee proven by Hannan.

Blackwell thus found an intriguing connection between repeated vector-payoff games and low-regret learning, a connection that we shall explore in greater detail in the present work. Indeed, we will show that the relationship goes much deeper than Blackwell had originally supposed. We prove that, in fact, Blackwell's Approachability Theorem is equivalent, in a very strong sense, to no-regret learning, for the particular setting of so-called ``Online Linear Optimization''. Precisely, we show that any no-regret algorithm can be converted into an algorithm for Approachability and vice versa. This is algorithmic equivalence is achieved via the use of \emph{conic duality}: if our goal is low-regret learning in a cone $K$, we can convert this into a problem of approachability of the dual cone $K^0$, and vice versa.

This equivalence provides a range of benefits and one such is ``calibrated forecasting''. The goal of a calibrated forecaster is to ensure that sequential probability predictions of repeated events are
``unbiased'' in the following sense: when the weatherman says ``30\% chance of rain'', it should actually rain roughly three times out of ten. The problem of calibrated forecasting was reduced to Blackwell's Approachability Theorem by Foster~\cite{foster_proof_1999}, and a handful of other calibration techniques have been proposed, yet none have provided any efficiency guarantees on the strategy. Using a similar reduction from calibration to approachability, and by carefully constructing the reduction from approachability to online linear optimization, we achieve the first efficient calibration algorithm.

\paragraph{Related work}
There is by now vast literature on all three main topics of this paper: approachability, online learning and calibration, see \cite{cesa-bianchi_prediction_2006} for an excellent exposition.  The relation between the three areas is not as well-understood. 

Blackwell himself noted that approachability implies no regret algorithms in the discrete setting. However, as we show hereby, the full power of approachability extends to a much more general framework of online linear optimization, which has only recently been explored (see \cite{Hsurvey10} for a survey) and shown to give the first efficient algorithms for a host of problems (e.g. \cite{AbernethyHR08,even-dar_online_2009}). Perhaps more significant, we also prove the reverse direction - online linear optimization exactly captures the power of approachability. Previously, it was considered by many to be strictly stronger than regret minimization.

Calibration is a fundamental notion in prediction theory and has found numerous applications in economics and learning. Dawid \cite{dawid82} was the first to define calibration, with numerous algorithms later given by Foster and Vohra \cite{foster_asymptotic_1998}, Fudenberg and Levine \cite{fudenberg_easier_1999}, Hart and Mas-Colell \cite{hart_simple_2000} and more. Foster has given a calibration algorithm based on approachability \cite{foster_proof_1999}.
There are numerous definitions of calibration in the literature, mostly asymptotic. In this paper we give precise finite-time rates of calibration and show them to be optimal. Furthermore, we give the first {\it efficient} algorithm for calibration: attaining $\epsilon$-calibration (formally defined later) required a running time of $poly(\frac{1}{\epsilon})$ for all previous algorithms, whereas our algorithm runs in time proportional to $\log \frac{1}{\epsilon}$.

\section{Preliminaries}
\subsection{Blackwell Approachability} \label{sec:blackwell_approach}

A \emph{vector-valued game} is defined by a pair of convex compact sets $\X \subset \reals^{n}, \Y \subset \reals^{m}$ and a biaffine mapping $\lb: \X \times \Y \to \reals^{d}$; that is, for any $\alpha \in [0,1]$ and any $\xv_1,\xv_2 \in \X$, $\yv_1, \yv_2 \in \Y$, we have $\lb(\alpha \xv_{1} + (1-\alpha)\xv_{2}, \yv) = \alpha\lb(\xv_{1},\yv) + (1-\alpha)\lb(\xv_{2},\yv)$, and $\lb(\xv,\alpha\yv_{1} + (1-\alpha)\yv_{2}) = \alpha\lb(\xv,\yv_{1}) + (1-\alpha)\lb(\xv,\yv_{2})$. We consider $\lb(\xv,\yv)$ to be the ``payoff vector'' when Player 1 plays strategy $\xv$ and Player 2 plays strategy $\yv$. We consider this game from the perspective of Player 1, whom we will often refer to as ``the player'', while we refer to Player 2 as ``the adversary''.

In a scalar-valued game, the natural question to ask is ``how much can a player expect to gain/lose (in expectation) against a worst-case adversary?'' With $d$-dimensional payoffs, of course, we don't have a notion of `more' or `less', and hence this question does not make sense. As Blackwell pointed out \cite{blackwell_analog_1956}, the natural question to consider is ``can we guarantee that the payoff vector lies in a given (convex) set S?'' Notice that this formulation dovetails nicely with the original goal in scalar-valued game, in the following way. Take any \emph{halfspace} $H \subset \reals^{d}$, where $H$ is parameterized by a vector $\vv$ and a constant $c$, namely $H = \{\zv \in \reals^{d} : \zv \cdot \vv \geq c\}$. Then the question ``can we guarantee that $\lb(\cdot,\cdot)$ lies in H?'' is equivalent to ``can the player expect to gain at least $c$ in the scalar-valued game defined by $\ell'(\xv,\yv) := \lb(\xv,\yv)\cdot\vv$?''

Given that we would like to receive payoff vectors that lie within $S$, let us define three separate notions of achievement towards this goal.
\begin{defn}
	Let $S$ be a convex set of $\reals^{d}$.
	\begin{itemize}
		\item We say that a set $S$ is \emph{satisfiable} if there exists a strategy $\xv \in \X$ such that for any $\yv \in \Y$, $\lb(\xv,\yv) \in S$.
		\item We say that a set is $S$ \emph{halfspace-satisfiable} if, for any halfspace $H \supseteq S$, $H$ is satisfiable.
		\item We say that a set is $S$ is \emph{response-satisfiable} if, for any $\yv \in \Y$, there exists a $\xv_{\yv} \in \X$ such that $\lb(\xv_{\yv},\yv) \in S$.
	\end{itemize}
\end{defn}

Among these three conditions the first, satisfiability, is the strongest. Indeed, it says that the player has an \emph{oblivious} strategy which always provides the desired guarantee, namely that the payoff is in $S$. The second condition, response-satisfiability, is much weaker and says we can achieve the same guarantee provided we observe the opponent's strategy in advance. We will also make use of the final condition, halfspace-satisfiability, which is also a weak condition, although we shall show it is equivalent to response-satisifiability.

Of course, a scalar-valued game is a particular case of a vector-valued game. What is interesting is that, for this special case, the condition of satisfiability is in fact no stronger than response-satisfiability for the case when $S$ has the form $[c,\infty)$. Indeed, this fact can be view as the celebrated Minimax Theorem.
\begin{thm}[Von Neumann's Minimax Theorem \cite{von_neumann_theory_1947}]
	For $\X$ and $\Y$ the $n$-dimensional and $m$-dimensional probability simplexes, and with scalar-valued $\lb(\cdot,\cdot)$, the set $S = [c,\infty)$ is satisfiable if and only if it is response-satisfiable.
	\label{thm:minimax}
\end{thm}

We will also make use of a more general version of the Minimax Theorem, due to Maurice Sion.

\begin{thm}[Sion, 1958 \cite{sion1958general}] \label{thm:sion}
	Given convex compact sets $\X \subset \reals^{n}, \Y \subset \reals^{m}$, and a function $f : \X \times \Y \to \reals$ convex and concave in its first and second arguments respectively, we have
	\[
		\inf_{\xv \in \X}\sup_{\yv \in \Y} f(\xv,\yv) =\sup_{\yv \in \Y} \inf_{\xv \in \X} f(\xv,\yv)
	\]
\end{thm}
One might hope that the analog of Theorem~\ref{thm:minimax} for vector-valued games would also hold true. Unfortunately, this is not the case. Consider the following easy example: $\X = \Y := [0,1]$, the payoff is simply $\lb(x,y) := (x,y)$ for $x,y \in [0,1]$, and the set in question is $S := \{ (z,z) \; \forall z\in[0,1]\}$. Response-satisfiability is easy to establish, simply use the response strategy $x_{y} = y$. But satisfiability can not be achieved: there is certainly no generic $x$ for which $(x,y) \in S$ for all $y$.

At first glance, it seems unfortunate that we can not achieve a similar notion of duality for games with vector-valued payoffs. What Blackwell showed, however, is that the story is not quite so bad: we \emph{can} obtain a version of Theorem~\ref{thm:minimax} for a weaker notion of satisfiability. In particular, Blackwell proved that, so long as we can play this game repeatedly, then there exists an adaptive algorithm for playing this game that guarantees satisfiability for the \emph{average} payoff vector in the limit. Blackwell coined the term \emph{approachability}.
\begin{defn}
	Consider a vector-valued game $\lb(\cdot,\cdot)$ and a convex set $S$. Imagine we have some ``learning'' algorithm $\A$ which, given a sequence $\yv_{1}, \yv_2, \ldots \in \Y$, produces a sequence $\xv_1,\xv_2,...$ via the rule $\xv_{t} \leftarrow \A(\yv_{1}, \yv_{2}, \ldots, \yv_{t-1})$.  For any $T$ define\footnote{We may simply write $D_T(\A)$ when $S$ and $\yv_1, \ldots, \yv_T$ are clear from context.} the {\it distance} of $\A$ to be
	\[
		D_T(\A;S, \yv_1, \ldots, \yv_T) \equiv \dist\left(\frac 1 T \sum_{t=1}^T \ell(\xv_t, \yv_t), S \right),
	\]
	where here $\dist(,)$ mean the usual notion $\ell_2$-distance between a point and a set.
	For a given vector-valued game and a convex set $S$, we say that $S$ is \emph{approachable} if there exists a learning algorithm $\A$ such that,
	\[
		\mathop{\lim \sup}_{T \to \infty} D_T(\A;S, \yv_1, \ldots, \yv_T) = 0 \quad \quad \textnormal{for any sequence }\yv_{1}, \yv_2, \ldots \in \Y
	\]
\end{defn}

Approachability is a curious property: it allows the player to repeat the game and learn from his opponent, and only requires that the average payoff satisfy the desired guarantee in the long run. Blackwell showed that response-satisfiability, which does not imply satisfiability, does imply approachability.

\begin{thm}[Blackwell's Approachability Theorem \cite{blackwell_analog_1956}] \label{thm:blackwell}
	Any closed convex set $S$ is approachable if and only if it is response-satisfiable.
\end{thm}

This version of the theorem, which appears in Evan-Dar et al. \cite{even-dar_online_2009}, is not the one usually attributed to Blackwell, although this is essentially one of his corollaries. His main theorem states that halfspace-satisfiability, rather than response-satisfiability, implies satisfiability. However, these two weaker satisfiability conditions are equivalent:
\begin{lem} \label{lem:halfresp}
	Given a biaffine function $\lb(\cdot,\cdot)$ and any closed convex set $S$, $S$ is response-satisfiable if and only if it is halfspace-satisfiable.
\end{lem}

\begin{proof} 
	We will show each direction separately.
	\begin{itemize}
		\item[$\Longrightarrow$] Assume that $S$ is response-satisfiable. Hence, for any $\yv$ there is an $\xv_\yv$ such that $\lb(\xv_\yv,\yv) \in S$. Now take any halfspace $H \supset S$ parameterized by $\thb,c$, that is $H = \{ \zv : \langle \thb, \zv \rangle \leq c\}$. Then let us define a scalar-valued game with payoff function
		\[
			f(\xv,\yv) = \langle \thb, \lb(\xv,\yv) \rangle.
		\]
		Notice that $H \supset S$ implies that $\thb \cdot \zv \leq c$ for all $\zv \in S$. Combining this with the definition of response-satisfiability, we see that
		\[
			\sup_{\yv \in \Y} \inf_{\xv \in \X} f(\xv,\yv) \leq \sup_{\yv \in \Y} f(\xv_\yv,\yv)
			\leq c.
		\]
		By Sion's Theorem (Theorem \ref{thm:sion}), it follows that $\inf_{\xv \in \X} \sup_{\yv \in \Y} f(\xv,\yv) \leq c$. By compactness of $\X$, we can choose the minimizer $\xv^*$ of this optimization. Notice that, for any $\yv \in \Y$, we have that $f(\xv^*,\yv) \leq c$ by construction, and hence $\lb(\xv^*,\yv) \in H$. $H$ is thus satisfiable, as desired.
		\item[$\Longleftarrow$] Assume that $S$ is not response-satisfiable. Hence, there must exists some $\yv_0$ such that $\lb(\xv,\yv_0) \notin S$ for every $\xv$. Consider the set $U := \{ \lb(\xv,\yv_0) \text{ for all }\xv \in X\}$ and notice that $U$ is convex since $\X$ is convex and $\lb(\cdot,\yv_0)$ is affine. Furthermore, because $S$ is convex and $S \cap U = \emptyset$ by assumption, there must exist some halfspace $H$ dividing the two, that is $S \subset H$ and $H \cap U = \emptyset$. By construction, we see that for any $\xv$, $\lb(\xv,\yv_0) \notin H$ and hence $H$ is not satisfiable. It follows immediately that $S$ is not halfspace-satisfiable.
	\end{itemize}
\end{proof}
For the sake of simplicity, and for natural connection to the Minimax Theorem, we prefer Theorem~\ref{thm:blackwell}. However, for certain results, it will be preferable to appeal to the halfspace-satisfiablility condition instead.

\subsection{Online Linear Optimization}

In the setting Online Linear Optimization, the ``learner'' makes decisions from a bounded convex decision set $\K$ in some Hilbert space. On each of a sequence of rounds, the decision maker chooses a point $\xv_t\in \K$, and is then given a linear cost function $\fv_t \in \F$, where $\F$ is some bounded set of cost functions, and cost $\langle \fv_t, \xv_t \rangle$ is paid. The standard measure of performance in this setting, called regret, is defined as follows.
\begin{defn}
	The \emph{regret} of learning algorithm $\L$ is
	\[
		\text{Regret}_T(\L) = \max_{\fv_1, \ldots, \fv_T \in \F} \; \left[\sum_{t=1}^T \langle \fv_t, \xv_t \rangle - \min_{\xv \in K} \sum_{t=1}^T \langle \fv_t, \xv \rangle\right]
	\]
We say that $\L$ is a \emph{no-regret} learning algorithm when it holds that $ \textnormal{Regret}_T(\L) = O(\sqrt{T}) = o(T)$.
\end{defn}
We state a well-known result:
\begin{thm}
For any bounded decision set $\K \subset \H$, there exists a no-regret algorithm on $\K$.
\end{thm}
Later in this paper, we shall use the Gradient Descent algorithm of Zinkevich~\cite{zinkevich_online_2003}. Ultimately, our goal will be to show that this theorem is equivalent to Theorem~\ref{thm:blackwell}.

\subsection{Convex cones in Hilbert space}

\begin{defn}
	A set $X \subset \reals^d$ is a \emph{cone} if it is closed under addition and multiplication by nonnegative scalars.
	Given any set $K \subset \H$, define $\cone(K) := \{ \alpha \xv : \alpha \in \reals_+, \xv \in K \}$,
	which is a cone in $\H$. Also, given any set in Hilbert space $C \subset \H$, we can define the \emph{polar cone} of $C$ as
	\[
		C^0 := \{ \thb \in \H : \langle \thb, \xv \rangle \leq 0 \text{ for all } \xv \in C \}
	\]
\end{defn}

We state a few simple facts on convex sets:
\begin{lem}\label{lem:tools}
	If $C$ is a convex cone then (1) $(C^0)^0 = C$ and (2) supporting hyperplanes in $C^0$ correspond to points $\xv \in C$, and vice versa. That is, given any supporting hyperplane $H$ of $C^0$, $H$ can be written exactly as $\{ \thb \in \reals^d : \langle \thb, \xv \rangle = 0 \}$ for some vector $\xv \in C$ that is unique up to scaling.
\end{lem}

%
%

The distance to a cone can conveniently be measure via a ``dual formulation,'' as we now show.
\begin{lem} \label{lem:distance-cone}
For every convex cone $C$ in Hilbert space
\begin{equation}\label{eq:distbnd}
	\dist(\xv,C)  =  \max_{\thb \in C^0, \|\thb\| \leq 1} \langle\thb , \xv\rangle
\end{equation}
\end{lem}

We need to measure distance to $\K$ after we make it into a cone via ``lifting''.
\begin{lem} \label{lem:distance-set}
Consider a convex set $\K \subseteq \H$ in Hilbert space and $\x \notin \K$. Let $\|\K\| := \max_{\yv \in \K} \|\yv\|$. Define $C$ to be the cone generated by the \emph{lifting} of $\K$, that is $C = \cone(\{1\} \oplus \K)$. Then
\begin{equation}
	\dist(1 \oplus \xv,C) \leq \dist(\xv,\K)  \leq    ( 1 + \|\K\|) \cdot \dist(1 \oplus \xv,C)
\end{equation}
\end{lem}

\section{Duality of Approachability and Low-Regret Learning} \label{sec:duality}

We recall the notion of an approachability algorithm $\A$ from Section~\ref{sec:blackwell_approach}. Formally, we imagine $\A$ as a function that observes a sequence of opponent plays $\yv_1, \ldots, \yv_{t-1} \in \Y$ and  chooses $\xv_{t} \leftarrow \A(\yv_{1}, \yv_{2}, \ldots, \yv_{t-1})$ from $\X$, with the goal that $\frac 1 T \sum_{t=1}^T \lb(\xv_t,\yv_t)$ approaches a convex set $S$. The convex decision sets $\X,\Y$, the payoff function $\lb(\cdot,\cdot)$, and the set $S$ are all known in advance to $\A$, and hence we may also write $\A_{\lb, S}$ for the algorithm tuned for these particular choices.

Equivalently, we consider a no-regret algorithm $\L$ as a function that observes a sequence of linear cost functions $\fv_1, \ldots, \fv_{t-1}$ and returns a point $\xv_t \leftarrow \L(\fv_1, \ldots, \fv_{t-1})$ from the decision set $\K$. The goal here is to achieve the regret $\sum_{t=1}^T \langle \fv_t, \xv_t \rangle - \min_{\xv \in \K} \sum_{t=1}^T \langle \fv_t, \xv \rangle$ that is sublinear in $T$. The bounded convex set $\K$ is known to the algorithm in advance, and hence we may write $\A_\K$ for the algorithm tuned for this particular set $\K$.

We now prove two claims, showing the equivalence of Blackwell and Online Linear Optimization. 	Precisely what we will show is the following. Assume we are given (A) an instance of an Online Linear Optimization problem and (B) an algorithm that achieves the goals of Blackwell's Approachability Theorem. Then we shall show that we can convert the algorithm for (B) to achieve a no-regret algorithm for (A).

\begin{algorithm}
	\caption{Reduction of Approachability Alg. $\A$ to Online Linear Optimization Alg. $\L$} \label{alg:bwa_to_lra}
	\begin{algorithmic}
		\STATE Input: Convex decision set $\K \subset \reals^d$
		\STATE Input: Sequence of cost functions $\fv_1, \fv_2, \ldots, \fv_T \in \reals^d$
		\STATE Input: Approachability algorithm $\A$
		\STATE Set: Two-player vector-payoff game $\lb : \K \times \reals^d \to \reals^{d+1}$ as $\lb(\xv,\fv) = \langle\fv, \xv\rangle \oplus -\fv$
		\STATE Set: Approach set $S := \cone(1 \oplus \K)^0$
		\FOR{$t=1, \ldots, T$}
			\STATE Let: $\L(\fv_1, \ldots, \fv_{t-1}) := \A_{\lb,S}(\fv_1, \ldots, \fv_{t-1})$
			\STATE Receive: cost function $\fv_t$
		\ENDFOR
	\end{algorithmic}
\end{algorithm}

\begin{lem}\label{lem:resp_sat_possible}
	For $S$ defined in Algorithm~\ref{alg:bwa_to_lra}, there exists an approachability algorithm $\A$ for $S$; that is, $D_T(\A;S) \to 0$ as $T \to \infty$.
\end{lem}


\begin{prop}\label{prop:olo_to_blackwell}
The reduction defined in Algorithm~\ref{alg:bwa_to_lra}, for any input algorithm $\A$, produces an OLO algorithm $\L$ such that $\frac{\regret(\L)}{T} \leq (1 + \|\K\|) {D_T(\A)}$.
\end{prop}


Now onto the second reduction. The construction in Algorithm~\ref{alg:lra_to_bwa} attempts the following. Assume we are given (A) an instance of a vector-payoff game and an approach set $S$ and (B) a low-regret OLO algorithm. Then we shall show that we can convert the algorithm for (B) to achieve approachability for (A).

\begin{algorithm}
	\caption{Conversion of Online Linear Optimization Alg. $\L$ to Approachability Alg. $\A$} \label{alg:lra_to_bwa}
	\begin{algorithmic}
		\STATE Input: Convex compact decision sets $\X \subset \reals^n$ and $\Y \subset \reals^m$
		\STATE Input: Biaffine vector-payoff function $\lb(\cdot,\cdot) : \X \times \Y \to \reals^d$
		\STATE Input: Approaching Set $S \subset \{1\} \times \reals^{d-1}$
		\STATE Input: Online Linear Optimization algorithm $\L$
		\STATE Set: $\K = \cone(S)^0 \cap B_1$
		\FOR{$t=1, \ldots, T$}
		\STATE Query: $\thb_t \leftarrow \L_{\K}(\fv_1, \ldots, \fv_{t-1})$, where $\fv_s \leftarrow -\lb(\xv_s,\yv_s)$
		\STATE Compute: $\xv_t \in \X$ so that $\langle\thb_t,\lb(\xv_t,\yv)\rangle \leq 0$ for any $\yv \in \Y$ \quad \quad  // \emph{Halfspace oracle}
		\STATE Let: $\A(\yv_1, \ldots, \yv_{t-1}) := \xv_t$
		\STATE Receive: $\yv_t \in \Y$
		\ENDFOR
	\end{algorithmic}
\end{algorithm}

\begin{prop} \label{claim:regret2approachability}
The reduction in Algorithm~\ref{alg:lra_to_bwa} produces an approachability algorithm $\A$ with distance bounded by
$$D_T(\A) \leq (1 + \|S\|)\frac{\regret(\L)}{T}$$
as long as $S$ is halfspace-satisfiable with respect to $\lb(\cdot,\cdot)$.
\end{prop}


\section{Efficient Calibration via Approachability and OLO}

Imagine a sequence of binary outcomes, say `rain' or `shine' on a given day, and imagine a forecaster, say the weatherman, that wants to predict the probability of this outcome on each day. A natural question to ask is, on the days when the weatherman actually predicts ``30\% chance of rain'', does it actually rain (roughly) 30\% of the time? This exactly the problem of \emph{calibrated forecasting} which we now discuss.

There have been a range of definitions of calibration given throughout the literature, some equivalent and some not, but from a computational viewpoint there are significant differences. We thus give a clean definition of calibration, first introduced by Foster~\cite{foster_proof_1999}, which is convenient to asses computationally.

We let $y_1, y_2, \ldots \in \{0,1\}$ be a sequence of outcomes, and $p_1, p_2, \ldots \in [0,1]$ a sequence of probability predictions by a forecaster. We define for every $T$ and every probability interval $[a,b]$, where $0 \leq a \leq b \leq 1$, the quantities
\[
	n_T(p,\epsilon) := \sum_{t=1}^T \I[p_t \in (p-\epsilon/2,p+\epsilon/2)], \quad \quad
	\rho_T(p,\epsilon) := \frac{\sum_{t=1}^T y_t \I[p_t \in (p-\epsilon/2,p+\epsilon/2)]}
	{n_T(p,\epsilon)}.
\]
The quantity $\rho_T(p-\epsilon/2,p+\epsilon/2)$ should be interpreted as the empirical frequency of $y_t = 1$, up to round $T$, on only those rounds where the forecaster's prediction was ``roughly'' equal to $p$. The goal of calibration, of course, is to have this empirical frequency $\rho_T(p,\epsilon)$ be close to the estimated frequency $p$, which leads us to the following definition.

\begin{defn}
Let the $(\ell_1,\epsilon)$-calibration rate for forecaster $\A$ be
	\[
		C_T^{\epsilon}(\A) = \sum_{i=0}^{\lfloor \epsilon^{-1} \rfloor} \frac{n_T(i\epsilon,\epsilon)}{T}\left|  i \epsilon - \rho_T(i\epsilon,\epsilon) \right| - \frac{\varepsilon}{2}
	\]
	We say that a forecaster is \emph{($\ell_1,\epsilon)$-calibrated} if $C_T^{\epsilon}(\A) = o(1)$. This in turn implies $\mathop{\lim\sup}_{T \to \infty} C_T^{\epsilon}(\A) = 0$.
\end{defn}

This definition emphasizes that we can ignore an interval $(p-\epsilon/2,p+\epsilon/2)$ in cases when our forecaster ``rarely'' makes predictions within this interval---more precisely, when we forecast within this interval with a frequency that is sublinear in $T$. Another important feature of this definition is the constant $\epsilon/2$ - which is an artifact of the discretization by $\epsilon$. This is the smallest constant which allows for $\mathop{\lim\sup}_{T \to \infty} C_T^{\epsilon}(\A) = 0$.

We given an equivalent and alternative characterization of this definition: let the {\emph calibration vector} at time $T$ denoted $c_T$ be given by:
$c_T(i) = \frac{n_T(i\epsilon,\epsilon)}{T}\left|  i \epsilon - \rho_T(i\epsilon,\epsilon) \right| $
\begin{claim} \label{claim:caldistance}
The $(\ell_1,\epsilon)$-calibration rate is equal to the distance of the calibration vector to the \Lone ball of radius $\epsilon/2$:
$$C_T^{\epsilon} = \dist(c_T, B_1(\epsilon/2))$$
\end{claim}
\begin{proof}
Notice that:
\begin{equation*}
	\dist_1(\xv,B_{1}(\epsilon/2)) := \min_{\yv : \|\yv\|_1 \leq \epsilon/2} \| \xv - \yv \|_1 = -\epsilon/2 + \| \xv \|_1
\end{equation*}
where the second equality follows by noting that an optimally chosen $\yv$ will lie in the same quadrant as $\xv$.
\end{proof}

A standard reduction in the literature (see e.g. \cite{cesa-bianchi_prediction_2006}) shows that $\epsilon$-calibration and full calibration are essentially the same (in the sense
that an $\epsilon$-calibrated algorithm can be converted to a calibrated one). For simplicity we consider only $\epsilon$-calibration henceforth.

\subsection{Existence of Calibrated Forecaster via Blackwell Approachability}

A surprising fact is that it is possible to achieve calibration even when the outcome sequence $\{y_t\}$ is chosen by an adversary, although this requires a randomized strategy of the forecaster. Algorithms for calibrated forecasting under adversarial conditions have been given in Foster and Vohra \cite{foster_asymptotic_1998}, Fudenberg and Levine \cite{fudenberg_easier_1999}, and Hart and Mas-Colell \cite{hart_simple_2000}.

Interestingly, the calibration problem was reduced to Blackwell's Approachability Theorem in a short paper by Foster in 1999 \cite{foster_proof_1999}. Foster's reduction uses Blackwell's original theorem, proving that a given set is halfspace-satisfiable, in particular by providing a construction for each such halfspace. Here, we provide a reduction to Blackwell Approachability using the response-satisfiability condition, i.e. via Theorem~\ref{thm:blackwell}, which is both significantly easier and more intuitive than Blackwell\footnote{A similar existence proof was discovered concurrently by Mannor and Stoltz \cite{mannor2009geometric}}. We also show, using the reduction to Online Linear Optimization from the previous section, how to achieve the most efficient known algorithm for calibration by taking advantage of the Online Gradient Descent algorithm of Zinkevich \cite{zinkevich_online_2003}, using the results of Section~\ref{sec:duality}.

We now describe the construction that allows us to reduce calibration to approachability. For any $\epsilon > 0$ we will show how to construct an $(\ell_1,\epsilon)$-calibrated forecaster. Notice that from here, it is straightforward to produce a well-calibrated forecaster \cite{foster_asymptotic_1998}. For simplicity, assume $\epsilon = 1/m$ for some positive integer $m$. On each round $t$, a forecaster will now randomly predict a probability $p_t \in \{0/m, 1/m, 2/m, \ldots, (m-1)/m, 1\}$, according to the distribution $\wv_t$, that is $\text{Pr}(p_t = i/m) = w_t(i)$. We now define a vector-valued game. Let the player choose $\wv_t \in \X := \Delta_{m+1}$, and the adversary choose $y_t \in \Y := [0,1]$, and the payoff vector will be
\begin{equation}\label{eq:calib_game}
	\lb(\wv_t,y_t) := \left\langle \wv_t(0)\left(y_t - \frac 0 m\right), \wv_t(1)\left(y_t - \frac 1 m\right), \ldots, \wv_t(m)(y_t - 1) \right\rangle
\end{equation}
\begin{lem}
	Consider the vector-valued game described above and let $S$, the $\ell_1$ ball of radius $\epsilon/2$. If we have a strategy for choosing $\wv_t$ that guarantees approachability of $S$, that is $\frac 1 T \sum_{t=1}^T \lb(\wv_t,y_t) \to S$, then a randomized forecaster that selects $p_t$ according to $\wv_t$ is $(\ell_1,\epsilon)$-calibrated with high probability.
\end{lem}
The proof of this lemma is straightforward, and is similar to the construction in Foster \cite{foster_proof_1999}. The vector $\frac 1 T \sum_{t=1}^T \lb(\wv_t,y_t)$ is simply the expectation of $(\ell_1,\epsilon)$-calibration vector at $T$. Since each $p_t$ is drawn independently, by standard concentration arguments we can see that if $\frac 1 T \sum_{t=1}^T \lb(\wv_t,y_t)$ is close to the \Lone ball of radius $\epsilon/2$, then the $(\ell_1,\epsilon)$-calibration vector is close to the $\epsilon/2$ ball with high probability.

We can now apply Theorem~\ref{thm:blackwell} to prove the existence of a calibrated forecaster.
\begin{thm} \label{thm:calapp}
	For the vector-valued game defined in \eqref{eq:calib_game}, the \Lone\ ball of radius $\epsilon/2$ is response-satisfiable and, hence, approachable.
\end{thm}
\begin{proof}
	To show response-satisfiability, we need only show that, for every strategy $y \in [0,1]$ played by the adversary, there is a strategy $\wv \in \Delta_m$ for which $\lb(\wv,y) \in S$. This can be achieved by simply setting $i$ so as to minimize $|i\epsilon - y|$, which can always be made smaller than $\epsilon/2$. We then choose our distribution $\w \in \Delta_{m+1}$ to be a point mass on $i$, that is we set $w(i) = 1$ and $w(j) = 0$ for all $j \ne i$. Then $\lb(\wv,y)$ is identically 0 everywhere except the $i$th coordinate, which has the value $y - i/m$. By construction, $y - i/m \in [-1/m,1/m]$, and we are done.
\end{proof}

\subsection{Efficient Algorithm for Calibration via Online Linear Optimization}

We now show how the results in the previous Section lead to the first efficient algorithm for calibrated forecasting. The previous theorem provides a natural existence proof for Calibration, but it does not immediately provide us with a simple and efficient algorithm. We proceed according to the reduction outlined in the previous section to prove:

\begin{thm} \label{thm:calib1}
There exists a $(\ell_1,\epsilon)$-calibration algorithm that runs in time $O(\log \frac{1}{\varepsilon})$ per iteration and satisfies $C_T^{\epsilon} = O\left(\frac{1}{\sqrt{\epsilon T}}\right)$
\end{thm}

The reduction developed in Proposition~\ref{claim:regret2approachability} has some flexibility, and we shall modify it for the purposes of this problem. The objects we shall need, as well as the required conditions, are as follows:
\begin{enumerate}
	\item A convex set $\K$
	\item An efficient learning algorithm $\A$ which, for any sequence $\fv_1, \fv_2, \ldots$, can select a sequence of points $\thb_1, \thb_2, \ldots \in \K$ with the guarantee that $\sum_{t=1}^T \langle \fv_t, \thb_t \rangle - \min_{\thb \in \K} \sum_{t=1}^T \langle \fv_t, \thb \rangle = o(T)$. For the reduction, we shall set $\fv_t \leftarrow -\lb(\wv_t,y_t)$.
	\item An efficient oracle that can select a particular $\wv_t \in \X$ for each $\thb_t \in \K$ with the guarantee that
	\begin{equation} \label{eq:dist_vs_reg}
		\dist\left(\frac 1 T \sum_{t=1}^T \lb(\wv_t,y_t), S \right) \leq \frac 1 T \left( \sum_{t=1}^T \langle -\lb(\wv_t,y_t), \thb_t \rangle - \min_{\thb \in \K} \sum_{t=1}^T \langle -\lb(\wv_t,y_t), \thb \rangle\right)
	\end{equation}
	where the function $\dist()$ can be with respect to any norm.
\end{enumerate}

\paragraph{The Setup}
Let $\K = B_{\infty}(1) =  \{ \thb \in \reals^d : \|\thb\|_\infty \leq 1\}$ be the unit cube. This is an appropriate choice because we can write the \Lone distance to $B_{1}(\epsilon/2)$ (the \Lone ball of radius $\epsilon/2$) as
\begin{equation} \label{eq:l1distbnd}
	\dist_1(\xv,B_{1}(\epsilon/2)) := \min_{\yv : \|\yv\|_1 \leq \epsilon/2} \| \xv - \yv \|_1 = -\epsilon/2 + \| \xv \|_1 =  - \epsilon/2 - \min_{\thb: \|\thb\|_\infty \leq 1}\langle -\xv, \thb \rangle,
\end{equation}
where the second equality follows by noting that an optimally chosen $\yv$ will lie in the same quadrant as $\xv$.
Furthermore, we shall construct our oracle mapping $\thb \mapsto \wv$ with the following guarantee: $\langle \lb(\wv,y), \thb \rangle \leq \epsilon/2$ for any $y$. Using this guarantee, and if we plug in $\xv = \frac 1 T \sum_{t=1}^T \lb(\wv_t,y_t)$ \eqref{eq:l1distbnd}, we arrive at:
\begin{eqnarray*}
	\dist_1\left(\frac{\sum_{t=1}^T \lb(\wv_t,y_t)}{T},B_{1}(\epsilon/2)\right) 
	& = & -\epsilon/2 - \min_{\thb: \|\thb\|_\infty \leq 1} \left\langle \frac{-\sum_{t=1}^T \lb(\wv_t,y_t)}{T}, \thb \right\rangle \\
	& \leq & \frac 1 T \left( \sum_{t=1}^T \langle -\lb(\wv_t,y_t), \thb_t \rangle - \min_{\thb \in \K} \sum_{t=1}^T \langle -\lb(\wv_t,y_t), \thb \rangle\right)
\end{eqnarray*}
This is precisely the necessary guarantee \eqref{eq:dist_vs_reg}.

\paragraph{Constructing the Oracle} We now turn our attention to designing the required oracle in an \emph{efficient} manner. In particular, given any $\thb$ with $\|\thb\|_\infty \leq 1$ we must construct $\wv \in \Delta_{m+1}$ so that $\langle \ell(\wv,y), \thb \rangle \leq \epsilon/2$ for any $y$. The details of this oracle are given in Algorithm~\ref{alg:oracle}.
\begin{algorithm}
	\caption{Constructing $\wv$ from $\thb$} \label{alg:oracle}
	\begin{algorithmic}
		\STATE Input: $\thb$ such that $\|\thb\|_\infty \leq 1$
		\IF{$\thb(0) \leq 0$}
		\STATE $\wv \leftarrow \delta_0$ \quad // That is, choose $\wv$ to place all weight on the 0th coordinate
		\ELSIF{$\theta(m) \geq 0$}
		\STATE $\wv \leftarrow \delta_m$ \quad // That is, choose $\wv$ to place all weight on the last coordinate
		\ELSE
		\STATE Binary search $\thb$ to find coordinate $i$ such that $\thb(i) > 0$ and $\thb(i+1) \leq 0$
		\STATE $\wv \leftarrow \frac{\thb(i)^{-1}}{\thb(i)^{-1} - \thb(i+1)^{-1}} \delta_i + \frac{-\thb(i+1)^{-1}}{\thb(i)^{-1} - \thb(i+1)^{-1}} \delta_{i+1}$
		\ENDIF
		\STATE Return $\wv$
	\end{algorithmic}
\end{algorithm}
It is straightforward why, in the final \textbf{else} condition, there must be such a pair of coordinates $i,i+1$ satisfying the condition. We need not be concerned with the case that $\thb(i+1) = 0$, where we can simply define $\frac 0 \infty = 0$ and $\frac \infty \infty = 1$ leading to $\wv \leftarrow \delta_{i+1}$. It is also clear that, with the binary search, this algorithm requires at most $O(\log m) = O(\log 1/\epsilon)$ computation.

In order to prove that this construction is valid we need to check the condition that, for any $y \in \{0,1\}$, $\langle \lb(\wv,y), \thb \rangle \leq \epsilon/2$; or more precisely, $\sum_{i=1}^m \thb(i)\wv(i) \left(y - \frac i m\right) \leq \epsilon/2$.
Recalling that $m = 1/\epsilon$, this is trivially checked for the case when $\thb(1) \leq 0$ or $\thb(m) \geq 0$. Otherwise, we have
\begin{eqnarray*}
	\langle \lb(\wv,y), \thb \rangle
	& = & \thb(i) \frac{\thb(i)^{-1}}{\thb(i)^{-1} - \thb(i+1)^{-1}} \left(y - \frac i m\right) + \thb(i+1)\frac{-\thb(i+1)^{-1}}{\thb(i)^{-1} - \thb(i+1)^{-1}} \left(y - \frac {i+1} m\right)\\
	& = & \frac{1}{\thb(i)^{-1} - \thb(i+1)^{-1}} \frac 1 m
	\; \leq \; \frac{\max(|\thb(i)|, |\thb(i+1)|)} 2 \epsilon \; \leq \; \frac \epsilon 2
\end{eqnarray*}

\paragraph{The Learning Algorithm} The final piece is to construct an efficient learning algorithm which leads to vanishing regret. That is, we need to construct a sequence of $\thb_t$'s in the unit cube (denoted $B_\infty(1)$) so that
\[
	\sum_{t=1}^T \langle \lb_t, \thb_t \rangle - \min_{\thb \in B_\infty(1)} \sum_{t=1}^T \langle \lb_t, \thb \rangle = o(T),
\]
where $\lb_t := \lb(\wv_t,y_t)$. There are a range of possible no-regret algorithms available, but we use the one given by Zinkevich known commonly as Online Gradient Descent \cite{zinkevich_online_2003}. The details are given in Algorithm~\ref{alg:zink}.
\begin{algorithm}
	\caption{Online Gradient Descent} \label{alg:zink}
	\begin{algorithmic}
		\STATE Input: convex set $\K \subset \reals^d$
		\STATE Initialize: $\thb_1 = \mathbf{0}$
		\STATE Set Parameter: $\eta = O(T^{-1/2})$
		\FOR{$t=1, \ldots, T$}
		\STATE Receive $\lb_t$
		\STATE $\thb_{t+1}' \leftarrow \thb_t - \eta \lb_t$ \quad // Gradient Descent Step
		\STATE $\thb_{t+1} \leftarrow \text{Project}_2(\thb_{t+1}',\K)$ \quad // L2 Projection Step
		\ENDFOR
	\end{algorithmic}
\end{algorithm}
This algorithm can indeed be implemented efficiently, requiring only $O(1)$ computation on each round and $O(\min\{m,T\})$ memory. The main advantage is that the vectors $\lb_t$ are generated via our oracle above, and these vectors are \emph{sparse}, having only at most two nonzero coordinates. Hence, the Gradient Descent Step requires only $O(1)$ computation. In addition, the Projection Step can also be performed in an efficient manner. Since we assume that $\thb_t \in B_\infty(1)$, the updated point $\thb_{t+1}'$ can violate at most two of the L$\infty$ constraints of the ball $B_\infty(1)$. An $\ell_2$ projection onto the cube requires simply rounding the violated coordinates into $[-1,1]$. The number of non-zero elements in $\thb$ can increase by at most two every iteration, and storing $\thb$ is the only state that online gradient descent needs to store, hence the algorithm can be implemented with $O(\min\{T,m\})$ memory. We thus arrive at an efficient no-regret algorithm for choosing $\thb_t$.

\begin{proof}[Proof of Theorem \ref{thm:calib1}]
Here we have bounded the distance directly by the regret, using equation \eqref{eq:dist_vs_reg}, which tells us that the calibration rate is bounded by the regret of the online learning algorithm. Online Gradient Descent guarantees the regret to be no more than $DG\sqrt{T}$, where $D$ is the $\ell_2$ diameter of the set, and $G$ is the $\ell_2$-norm of the largest cost vector. For the ball $B_\infty(1)$, the diameter $D = \sqrt{\frac{1}{\epsilon}}$, and we can bound the norm of our loss vectors by $G = \sqrt 2$. Hence:
\begin{equation}
C_T^{\epsilon} \quad = \quad  \dist(c_T, B_1(\epsilon/2))
\quad \leq \quad \frac{\regret_T}{T}
\quad \leq \quad \frac{G D}{\sqrt{T}}
\quad = \quad  O\left(\frac{1}{\sqrt{\epsilon T}}\right)
\end{equation}
\end{proof}

\bibliographystyle{plain}
\bibliography{blackwell}

\appendix

\section{Proofs}

\begin{proof}[Proof of Lemma~\ref{lem:distance-cone}]
	We need two simple observations. Define $\pi_C(\xv)$ as the projection of $\xv$ onto $C$. Then clearly, for any $\xv$,
	\begin{align}
		& \dist(\xv,C) = \| \xv - \pi_C(\xv) \| \label{eq:disproj} \\
		& \langle \xv - \pi_C(\x), \yv \rangle \leq 0 \; \forall \yv \in C \text{ and hence } \xv - \pi_C(\x) \in C^0 \label{eq:projdual} \\
		\label{eq:projperp} & \langle \xv - \pi_C(\xv), \pi_C(\xv) \rangle = 0
	\end{align}
	Given any $\thb \in C^0$ with $\|\thb\| \leq 1$, since $\pi_C(\xv) \in C$ we have that
	\[
		\langle \thb, \xv \rangle \leq \langle \thb, \xv - \pi_C(\xv) \rangle \leq \|\thb \| \|\xv - \pi_C(\xv) \| \leq \|\xv - \pi_C(\xv) \|,
	\]
	which immediately implies that $\max_{\thb \in C^0, \|\thb\| \leq 1} \langle\thb , \xv\rangle \leq \dist(\xv,C)$. Furthermore, by selecting $\thb = \frac{\xv - \pi_C(\xv)}{\|\xv - \pi_C(\xv)\|}$ which has norm one and, by \eqref{eq:disproj}, is in $C^0$, we see that
	\[
		\max_{\thb \in C^0, \|\thb\| \leq 1} \langle\thb , \xv\rangle \geq \left\langle \frac{\xv - \pi_C(\xv)}{\|\xv - \pi_C(\xv)\|}, \xv \right \rangle =
		\left\langle \frac{\xv - \pi_C(\xv)}{\|\xv - \pi_C(\xv)\|}, \xv - \pi_C(\xv) \right \rangle
		= \| \xv - \pi_C(\xv)\|,
	\]
	which implies that $\max_{\thb \in C^0, \|\thb\| \leq 1} \langle\thb , \xv\rangle \geq \dist(\xv,C)$ and hence we are done.
%
%
\end{proof}

\begin{proof}[Proof of Lemma~\ref{lem:distance-set}]
Since $\dist(1 \oplus \xv, \{1\} \times \K) = \dist(\x,\K)$ and $\{1\} \oplus \K \subset C$, the first inequality follows immediately.

For the second inequality, let $ \yv = 1 \oplus \x $ and let $\w,\uv$ be the closest points to $\yv$ in $C$ and $\K$ respectively. Consider the plane determined by these three points, as depicted in figure \ref{fig:cone}.
 \begin{figure}[h!]
 \begin{center}
 \includegraphics[width=4.0in]{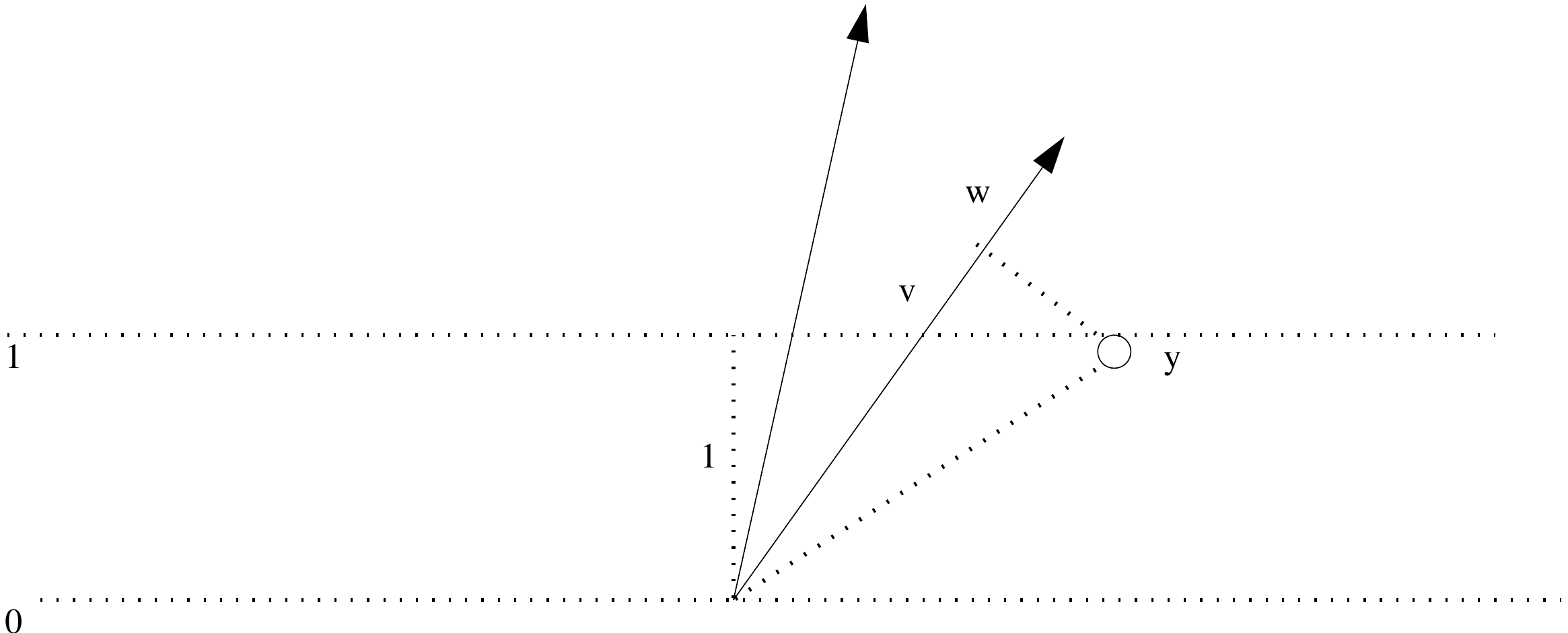}
 \end{center}
 \caption{The ratio of distances to $\K$ and the cone is the same as the ratio between $\|\uv\|$ and one. \label{fig:cone}}
 \end{figure}
Notice that, by triangle similarity, we have that
\[
	\|\uv\| = \frac{\|\uv\|}{\| 1 \oplus \zero \|} = \frac{\|\yv - \uv\|}{\|\yv - \wv\|} =  \frac{\dist(\yv,\{1\} \oplus \K)}{\dist(\yv,C)}
\]
Of course, $\uv \in \K$ and hence $\|\uv\| \leq \|\{1\} \oplus \K\| \leq 1 + \|\K\| $. The result follows immediately.
\end{proof}

\begin{proof}[Proof of Lemma~\ref{lem:resp_sat_possible}]
	The existence of an approachability algorithm is established by Blackwell's Approachability Theorem (Theorem~\ref{thm:blackwell}), as long as we can guaranteed the response-satisfiability condition. Precisely, we must show that, for any $\fv$, there is some $\xv_\fv \in \K$ such that $\lb(\xv_\fv,\fv) = \langle\fv, \xv_\fv \rangle \oplus -\fv \in \cone(1 \oplus \K)^0$.
	Recall that $\theta \in \cone(1 \oplus \K)^0$ if and only if $\langle \theta, \zv \rangle \leq 0$ for every $\zv \in \cone(1 \oplus \K)$. Observe that it suffices to restrict to only the set generating the cone, that is $\theta \in \cone(1 \oplus \K)^0$ if and only if $\langle 1 \oplus \xv', \zv \rangle \leq 0$ for each $\xv' \in \K$. Hence,
	\begin{eqnarray*}
		\lb(\xv,\fv) \in S & \Longleftrightarrow &  \langle \langle\fv, \xv\rangle \oplus -\fv, \zv \rangle \leq 0 \quad \forall \zv \in \cone(1 \oplus \K) \\
		& \Longleftrightarrow &  \langle \langle\fv, \xv\rangle  \oplus -\fv,  1 \oplus \xv' \rangle \leq 0 \quad \forall \xv' \in \K \\
		& \Longleftrightarrow &  \langle\fv, \xv\rangle \leq \langle\fv, \xv'\rangle  \quad \forall \xv' \in \K
	\end{eqnarray*}
	Of course, this can be achieved by setting $\xv = \arg\min_{\xv \in \K} \langle \fv, \xv \rangle$, and hence we are done.
\end{proof}

\begin{proof}[Proof of Proposition~\ref{claim:regret2approachability}]
	First notice that we require the halfspace-satisfiability condition for $S$ to ensure that the ``halfspace oracle'' in Algorithm~\ref{alg:lra_to_bwa} exists. Because we are selecting $\thb_t$ in $\K \subset \cone(S)^0$, $\thb_t$ defines a halfspace containing $S$ and hence we can use our halfspace oracle to find an $\xv_t$ satisying $\langle \thb_t,  \lb(\xv_t,\yv)\rangle \leq 0$ for every $\yv \in \Y$.
	
	To bound $D_T(\A)$, which is the distance between the point $\frac 1 T  \sum_{t=1}^T \lb(\xv_t, \yv_t)$ and the set $S$, we begin by instead bounding the distance to $\cone(S)$. We can immediately apply Lemma~\ref{lem:distance-cone} to obtain
	\begin{eqnarray}
		\dist\left( \frac 1 T  \sum_{t=1}^T \lb(\xv_t, \yv_t), \cone(S) \right)
		& = & \max_{\thb \in \K} \left\langle \frac 1 T  \sum_{t=1}^T \lb(\xv_t, \yv_t), \thb \right\rangle = \frac 1 T \max_{\thb \in \K}\left( -  \sum_{t=1}^T \langle \fv_t, \thb \rangle \right) \nonumber \\
		& \leq & \frac 1 T \left( \sum_{t=1}^T \langle \fv_t, \thb_t \rangle  - \min_{\thb \in \K} \sum_{t=1}^T \langle \fv_t, \thb \rangle \right)
		= \frac 1 T \text{Regret}_T(\A) \label{eq:dist_lt_regret}
	\end{eqnarray}
	where the first inequality follows by the halfspace oracle guarantee. Of course, if we let $S' \subset \reals^{d-1}$ be the set $S$ after removing the first coordinate, then we see by Lemma~\ref{lem:distance-set} that for any $\zv \in \reals^{d-1}$,
	\begin{equation}\label{eq:bound_on_dist}
		\dist(1\oplus \zv,S) = \dist(\zv,S') \leq (1 + \|S'\|) \dist(1 \oplus \zv,\cone(1 \oplus S')) \leq (1 + \|S \|) \dist(1 \oplus \zv,\cone(S)).
	\end{equation}
	By assumption, however, we can write  $\frac 1 T  \sum_{t=1}^T \lb(\xv_t, \yv_t) = 1 \oplus \zv$ for some $\zv \in \reals^{d-1}$. Combining this with equations \eqref{eq:dist_lt_regret} and \eqref{eq:bound_on_dist} finishes the proof.
\end{proof}

\begin{proof}[Proof of Proposition~\ref{prop:olo_to_blackwell}]%

	Applying Lemma~\ref{lem:distance-cone} to the definition of $D_T(\A)$ gives
	\begin{equation}
		D_T(\A) \equiv \dist\left(\frac 1 T \sum_{t=1}^T \ell(\xv_t, \fv_t), S\right) =
		 \max_{\w \in \cone(1 \oplus \K) \ , \ \|\w\| \leq 1 } \left \langle \frac 1 T \sum_{t=1}^T \ell(\xv_t, \fv_t), \w\right \rangle
	\end{equation}
	Notice that, in this optimization, we can assume w.l.o.g. that $\|\w\| = 1$, or $\w = \zero$. In the former case we can write $\w = \frac{1 \oplus \xv}{\|1 \oplus \xv\|}$ for some $\xv \in \K$, and we drop the latter case to obtain the inequality
	\begin{eqnarray*}
		D_T(\A)
		\geq \max_{\xv \in \K} \: \left \langle \frac 1 T \sum_{t=1}^T \ell(\xv_t, \fv_t), \frac{ 1 \oplus \xv }{\|1 \oplus \xv \|}  \right \rangle
		& = &  \frac 1 T  \max_{\xv \in \K}\frac {\left(\sum_{t=1}^T \langle\fv_t, \xv_t\rangle - \sum_{t=1}^T \langle\fv_t, \xv\rangle \right)} { \| 1 \oplus \xv\| } \\
		& \geq & \frac {  \frac 1 T\left(\sum_{t=1}^T \langle\fv_t, \xv_t\rangle - \sum_{t=1}^T \langle\fv_t, \xv^* \rangle \right)} {  \| 1 \oplus \xv^*\| } \geq  \frac{\frac 1 T\regret_T(\A)}{1 + \|\K\| },
	\end{eqnarray*}
	where we set $\xv^* := \arg\min_{\xv \in \K} \sum_{t=1}^T \langle\fv_t, \xv \rangle$.
\end{proof}

\section{A generalization of Blackwell to convex functions}

Consider the following generalization of Blackwell to functions. In analogy to Blackwell, let:
\begin{enumerate}
\item
A two-player game with functions as payoffs. For strategies $i,j$ we have that the payoff of the game $l(i,j) \in S$ is a function (rather than a vector as in Blackwell).
\item
$S$ -  set of functions $S: \{ f:\reals^d \mapsto \reals\} \subseteq \mathcal{F} $.
\item
"Halfspaces", which are characterized by $x \in \reals^d$. The halfspace $H_x$ contains all functions $f$ such that $f(x) \leq 0$.
\item
An oracle $O : H_x \mapsto p$ which maps a halfspace containing $S$, i.e. $\forall f \in S , f(x) \leq 0$, into a distribution over player strategies, such that the resulting loss function is contained inside the halfspace, i.e.
$$ \forall j \ , \ l(O(x), j) =  l(p,j) \in H_x$$
\item
When talking of approachability we need a distance measure. If we think of $f(x)$ as the "inner-product" between $f$ and $x$, then $K = \{ x | f(x) \leq 0  \ \ \forall f \in S \}$, which is the "dual" set to $S$, is the set of all hyperplanes containing $S$. Our distance measure between a function $f$ and $S$ is then taken to be the maximal inner product with any hyper-plane containing $S$:
$$ d(f,S) \equiv \max\{ 0, \max_{x \in K} f(x)\} $$
Note that this distance is zero for all members of $S$.
\end{enumerate}
Then:
\begin{thm} [Blackwell generalization for functions]
Given an Oracle as above, the set $S$ is approachable.
\end{thm}

The Blackwell method of proof is geometric in Nature, and it is not immediately clear how to generalize it to prove the above. However, using Online Convex Optimization, the proof is a simple generalization of the one in the previous sections:

\begin{proof}
Define the dual set to $S$ as
$$ K = \{ x | f(x) \leq 0  \ \ \forall f \in S \} $$
Define the gain function for iteration $t$ - $f_t$ - as $f_t = l(p_t,j_t)$ where $j_t$ is the adversary's strategy at iteration $t$, and $p_t$ is given by the oracle as the mixed user strategy for the hyperplane parameterized by $x_t$.

Hence, iteratively the OCO algorithm generates an $x_t \in K$, which is then fed to the Oracle to obtain $p_t = O(x_t)$, which in turn defines the cost function for this iteration $f_t = l(p_t,j)$. The OCO low-regret theorem guaranties us that
$$ \max_{x^*} \frac{1}{T} \sum_t f_t(x^*) - \frac{1}{T} \sum_t f_t (x_t) \leq \varepsilon_t  \mapsto 0 $$

By the guarantee provided by the oracle, we have that $f_t(x_t) \leq 0$, which combined with the above gives us:
$$ d(\bar{f},S) = \max_{x^*} \bar{f}(x^*)  = \max_{x^*} \frac{1}{T} \sum_t f_t(x^*)  \leq \varepsilon_t  \mapsto 0 $$
Which by our definition implies that the distance of the average gain function to the set $S$ converges to zero.

\end{proof}

\end{document}